\documentclass[letterpaper,10pt,conference]{ieeeconf}
\IEEEoverridecommandlockouts
\usepackage{xcolor}
\makeatletter
\let\NAT@parse\undefined
\makeatother
\usepackage[numbers,sort&compress]{natbib}
\usepackage[colorlinks=true,linkcolor=blue,citecolor=blue,urlcolor=blue]{hyperref} % hyperlinks to the references
\usepackage{amsthm}
\usepackage{amsmath}
\usepackage{amssymb}
\usepackage{algorithm}
\usepackage{algpseudocode}
\usepackage{tikz}
\usetikzlibrary{arrows.meta, positioning, calc}

\usepackage{graphicx}

\usepackage{float}
\usepackage{booktabs}
\usepackage{multirow}

\usepackage{soul}
\usepackage[normalem]{ulem}

\usepackage{booktabs}
\usepackage{tabularx}
\usepackage{caption}

\definecolor{papercream}{RGB}{250,247,230}

\hypersetup{colorlinks=true}

\definecolor{methodleap}{rgb}{0.000,0.450,0.700}
\definecolor{methodaligator}{rgb}{0.900,0.620,0.000}
\definecolor{methodcrocoddyl}{rgb}{0.800,0.400,0.650}

\providecommand{\methodname}{\text{LEAP}}

\newcommand{\cesar}[1]{
  {\textnormal{
    \color{blue}\textbf{(Cesar: #1)}
    }
  }
}

\long\def\ai_help#1{{\color{violet}#1}}

\providecommand{\W}{\mathcal{W}}

\providecommand{\A}{\mathcal{A}}

\providecommand{\R}{\mathcal{R}}

\providecommand{\Ii}{\mathcal{I}}

\providecommand{\E}{\mathcal{E}}

\providecommand{\W}{\mathcal{W}}

\providecommand{\RR}{\mathbb{R}}

\newtheorem{defn}{Definition}

\newtheorem{lem}[defn]{Lemma}
\newtheorem{prop}[defn]{Proposition}

\newtheorem{problem}[defn]{Problem}
\providecommand{\x}{x}
\providecommand{\xt}{x(t)}

\providecommand{\uu}{u}
\providecommand{\ut}{u(t)}

\providecommand{\qt}{q(t)}
\providecommand{\qdt}{\dot{q}(t)}

\providecommand{\vt}{v(t)}
\providecommand{\vdt}{\dot{v}(t)}

\providecommand{\sut}{\mathrm{s.t.}}
\newif\ifshowrevisions
\showrevisionsfalse
\definecolor{revold}{RGB}{190,30,45}   % struck-out old text
\definecolor{revnew}{RGB}{0,110,180}   % replacement / added text

\newcommand{\revise}[2]{%
  \ifshowrevisions
    {\color{revold}\sout{#1}}\,{\color{revnew}#2}%
  \else
    #2%
  \fi
}

\newcommand{\added}[1]{\ifshowrevisions{\color{revnew}#1}\else#1\fi}

\title{LEAP: Gradient Flow Augmented Lagrangian
Trajectory Optimization for Contact Constrained Robots}

\author{%
    C\'esar E. Ramos Chuquiure$^{1}$,
    Ryan Tsai$^{1}$,
    Yutong Zhou$^{1}$,
    Zihang Huang$^{2}$\\
    Vansh Thakur$^{1}$,
    Akshunn Trivedi$^{1}$,
    Zachary Lin$^{3}$,
    Pearl Lin$^{2}$\\
    Ruikai Yang$^{2}$,
    Challen Enninful Adu$^{1}$,
    Ram Vasudevan$^{1}$%
    \thanks{$^{1}$Department of Robotics,
    University of Michigan, Ann Arbor, MI, USA.}%
    \thanks{$^{2}$Computer Science and Engineering,
    University of Michigan, Ann Arbor, MI, USA.}%
    \thanks{$^{3}$Department of Electrical Engineering
    and Computer Science,
    University of Michigan, Ann Arbor, MI, USA.}%
    \thanks{This work was supported by AFOSR MURI FA9550-23-1-0400.}%
}

\makeatletter
\long\def\IEEEaftertitletext#1{\def\@IEEEaftertitletext{#1}}
\IEEEaftertitletext{%
    \vspace{-0.5\baselineskip}%
    \noindent\begin{minipage}{\textwidth}
        \centering
        \includegraphics[width=\textwidth]{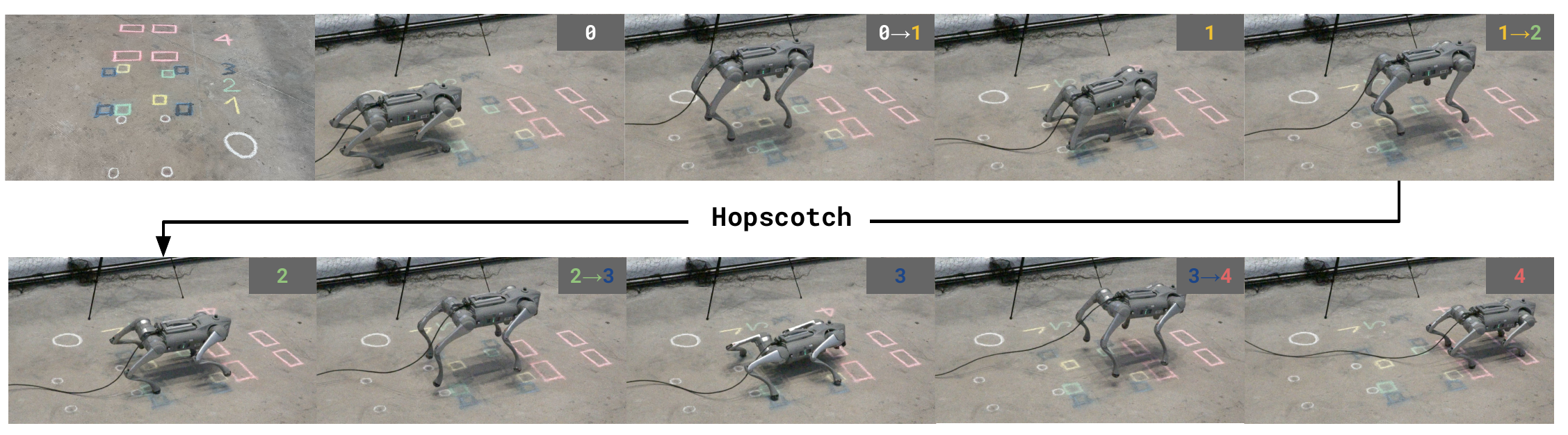}\par
        \refstepcounter{figure}\label{fig:cover}%
        \def\@captype{figure}%
    \@makecaption{Fig.~\thefigure}{%
        Hardware execution of a hopscotch trajectory generated by
        \methodname{}. 
        The quadruped performs successive jumps between prescribed foothold patterns, shown in the upper-left panel.
        Numbered frames indicate stages of the sequence, while
        $i \!\to\! i+1$ labels indicate transitions between them.
        Frames proceed chronologically from left to right across
        the top row, then the bottom row.
    }\par
    \end{minipage}%
    \vspace{0.4\baselineskip}%
}
\makeatother
\begin{document}
    \maketitle

    \begin{abstract}
        Generating dynamically feasible motions for legged and manipulating robots requires solving a large, nonlinear, non-convex optimal control problem (OCP) in which contact wrenches must satisfy friction-cone constraints throughout the motion, not only at isolated instants.
        Existing solvers are sensitive to the initial guess and scale poorly with the number of joints.
        This paper introduces \methodname{}, which transcribes the OCP with pseudospectral collocation and, instead of a Newton or SQP solve, integrates a primal--dual gradient flow of a Powell--Hestenes--Rockafellar augmented Lagrangian to compute a solution.
        Because the flow is stiff, this paper designs a linearly implicit integrator with a Gauss--Newton metric, momentum, and a merit-based step controller that keeps the augmented Lagrangian monotone.
        \revise{Across benchmark problems on five robots, from a 2-joint hopper to 32-joint humanoids, \methodname{} solves roughly an order of magnitude faster than a variety of state of the art methods, reaches lower constraint violation under a fixed time budget, and is less sensitive to hyper-parameters.}{Across locomotion benchmarks on five robot platforms, from planar hoppers to 32-joint humanoids, \methodname{} is the only method to solve every tested problem, achieves substantial speedups over Crocoddyl and Aligator on several benchmarks, and \revise{attains lower constraint violation across a range of compute budgets}{reaches trackable references at smaller compute budgets}.}
        The resulting trajectories are validated by tracking them in a high-fidelity simulator and on hardware.
        We will make the code public upon final acceptance of this paper.
    \end{abstract}

    % Sections
    % section 00 won't appear in the to-be-reviewed manuscript
    % it is something internal so it is easier to read
    % \input{sections/00_notation.tex} 
    \section{Introduction}

Robots interact with the world through contact: a humanoid walks by planting and lifting its feet, a manipulator pushes, grasps and releases objects, and a quadruped climbs by loading and unloading its limbs. 
Planning such motions means choosing joint torques and contact forces that respect the robot's dynamics, its actuation limits, and the physics of contact, which requires the contact wrench to stay inside a friction cone at every instant. 
Trajectory optimization formalizes this as an Optimal Control Problem (OCP) over hybrid dynamics: minimize a cost, subject to the equations of motion in each contact mode, and the constraints above.

Solving this OCP is difficult for three reasons. 
The dynamics are nonlinear and change discontinuously across contact modes.
The constraints, in particular the friction cones and the dynamics themselves, are nonlinear and non-convex, and any solver must decide how to enforce them and where in time.
And the resulting problem is non-convex in the decision variables, so the outcome depends on the initial guess, and a good one is often unavailable for a new behavior or a new robot. 
We refer to these as the dynamics, constraint, and initialization challenges.

State-of-the-art solvers exploit the temporal structure of the dynamics to make each iteration cheap. 
Differential Dynamic Programming (DDP) methods such as Crocoddyl \citep{croco} and Aligator \citep{aligator} combine quadratic approximations of the value function with multiple shooting, and direct methods such as acados \citep{acados, raptor} transcribe the OCP into a nonlinear program (NLP) with banded structure.
In both families each iteration is a Newton-type step, and this is what makes the constraint and initialization challenges persist. 
A Newton step is only meaningful near a solution, so these solvers rely on a good warm start and on line-search or trust-region globalization whose parameters must be tuned per problem. 
Constraints are handled by penalties, by projection onto a box, or by an augmented Lagrangian inside the Riccati recursion, and they are enforced only at discrete knots, so the friction cone can be violated between them. In Section \ref{sec:simulation_experiments} illustrates that, on the contact-constrained problems, these solvers scale poorly with the number of joints and their performance depends strongly on hyper-parameters.

Gradient flows (GF) offer a different way to handle the initialization challenge. 
Instead of taking Newton steps, the trajectory is deformed continuously in a pseudo-time until it becomes stationary for an action functional. 
Every step of the flow is a descent step that is well defined anywhere, not only near a solution, so the method tolerates poor initial guesses and admits continuation naturally.
\cite{aghf_og} and \cite{liu2019homotopy} introduced the Affine Geometric Heat Flow (AGHF), whose steady state solves the OCP, and \cite{fan2019midair}, \cite{aghf_free_bcs} and \cite{aghf_legged} extended it to free boundary conditions and to low-dimensional legged systems. 
\cite{phlame} and \cite{blaze} improved the numerics to reach higher-dimensional systems, and \cite{kim_el_aghf}\revise{and \textbf{magma} \cesar{arxiv magma?} generalized the Lagrangian of the action functional, the latter making explicit the link between the AGHF and gradient flows.}{extends the idea to extended Lagrangians.}
What these methods gain on initialization they give back on the other two challenges. 
Constraints enter as quadratic penalties, so they are never satisfied exactly, and raising the penalty to reduce the violation makes the flow stiffer. 
The stiff flow is then integrated with general-purpose solvers whose error control is unaware of the optimization structure, so the dynamics challenge reappears as a numerical one.

This paper introduces \methodname{}, a gradient-flow trajectory optimizer for robots in contact under a prescribed contact sequence, together with a numerical scheme designed for that flow. 
It keeps the robustness to initialization of gradient flows while resolving the aforementioned constraint and stiffness problems.
Our contributions are four-fold:
First, we utilize a  Powell--Hestenes--Rockafellar (PHR) augmented Lagrangian formulation and derive a two-time-scale primal--dual flow to enable exact constraint handling (Section \ref{sec:gradient_flow}). 
Second, we describe how to extend this approach to hybrid dynamical systems and illustrate how to ensure that wrench feasibility is satisfied for all time\revise{}{in Lemma \ref{lem:bezier_convex_contraints}.}
Third, we develop a linearly implicit solver to ensure that the augmented Lagrangian monotonically decreases along the numerical solution to the primal-dual flow (Section \ref{sec:numerical_scheme}). 
\revise{Finally, we compare \methodname{} against acados, Crocoddyl and Aligator on solve time, constraint violation under a time budget, and sensitivity to hyper-parameters across three tasks and seven robots with two to thirty-two joints with solution quality judged by a whole-body tracking controller in a high-fidelity simulator (Section \ref{sec:simulation_experiments}).}{
    Finally, we compare \methodname{} against Crocoddyl and Aligator on solve time and \revise{constraint violation}{terminal tracking accuracy} under a time budget across two tasks and five robots with two to thirty-two joints with solution quality judged by a whole-body tracking controller in a high-fidelity simulator (Section \ref{sec:simulation_experiments}).
}
A hopscotch trajectory is executed on a Unitree Go2 quadruped (Section~\ref{sec:hardware_experiments}).
    \section{Preliminaries}
\label{sec:preliminaries}

This section introduces the concepts necessary to construct the gradient-flow based trajectory optimization algorithm \methodname{}.

\subsection{Robot dynamics with a prescribed contact sequence}
\label{subsec:hybrid_robot_system}

Consider a robot with configuration $\qt \in \RR^{n_q}$, generalized velocity $\vt \in \RR^{n_v}$ related by $\qdt = N(\qt)\vt$, and state $\xt = (\qt, \vt) \in \RR^{n_x}$ on the horizon $I := [0, T]$.
Let $E$ be the set of possible contacts.
A contact $\ell \in E$ is described by a smooth holonomic residual $\phi_\ell : \RR^{n_q} \to \RR^{d_\ell}$, and is active while $\phi_\ell(\qt) = 0$.
A mode is a set of simultaneously active contacts, which we identify with the subset $j \subseteq E$.
Its stacked residual $\Phi_j := (\phi_\ell)_{\ell \in j}$ and contact Jacobian $J_{c,j}(q)$ are related by $\dot{\Phi}_j = J_{c,j}(q)\,v$.
Within mode $j$ the state obeys the differential algebraic equation
\begin{subequations}
\label{eqn:mode_dae}
\begin{align}
    \hspace*{-0.25cm} M(\qt)\vdt + C(\qt, \vt) &= B\ut + J_{c,j}^T(\qt)\,\lambda_j(t), \label{eqn:mode_dae_eom} \\
    \Phi_j(\qt) &= 0, \label{eqn:mode_dae_holonomic}
\end{align}
\end{subequations}
where $M$ is the mass matrix, $C$ collects Coriolis and gravity terms, $B$ is the actuation matrix, $\ut \in U$ is the control input and $\lambda_j(t)$ is the stacked contact wrench of the contacts in $j$.

We consider motions with a prescribed mode sequence $(j_0, \ldots, j_{M-1})$ and prescribed switching times $t_0 = 0 < t_1 < \cdots < t_{M-1} < t_M = T$, so that mode $j_i$ is active on $I_i := [t_i, t_{i+1}]$.
At a switching time $t_i$ the configuration is continuous, and the velocity is mapped by a reset $v(t_i^+) = R_{j_{i-1}, j_i}(q(t_i), v(t_i^-))$, which we model as a Newton impact \citep{robot_dynamics} when a contact becomes active and as the identity when a contact is released.
This is a hybrid system in the sense of \citet{burden2015} whose discrete transitions are fixed in advance, so no guard conditions are needed.

\subsection{Contact Sequence Robot OCP}
\label{subsec:contact_sequence_robot_ocp}

\begin{problem}
\label{prob:contact_sequence_ocp}
Given the mode sequence and switching times above, the Contact Sequence Robot OCP is
\begin{subequations}
\begin{align}
\underset{\x, \uu, \{ \lambda_i \} }{\inf} \, & \sum_{i=0}^{M-1} \int_{t_i}^{t_{i+1}} l_i(\xt, \ut, \lambda_i(t))~dt \label{eqn:ocp_integral} \\
\sut \quad & g_{i, k}(\xt, \ut, \lambda_i(t)) \leq 0 \quad \forall k \in \Ii_i,\ \forall t \in I_i, \label{eqn:ocp_inequality_constraints} \\
& g_{i, k}(\xt, \ut, \lambda_i(t)) = 0 \quad \forall k \in \E_i,\ \forall t \in I_i, \label{eqn:ocp_equality_constraints}
\end{align}
\end{subequations}
for $i \in \{0, \ldots, M-1\}$, where $\lambda_i$ abbreviates $\lambda_{j_i}$.
\end{problem}
The equality constraints $\E_i$ contain the dynamics \eqref{eqn:mode_dae}, the reset map at $t_i$, and the boundary conditions; the inequality constraints $\Ii_i$ contain the friction cone on $\lambda_i$ and bounds on $\x$ and $\uu$.
All constraints must hold for every $t$, not only at a finite set of times.

\subsection{Polynomial representations}
\label{subsec:interpolating_polynomials}

To transcribe Problem \ref{prob:contact_sequence_ocp} we represent each trajectory on an interval, rescaled to $[-1, 1]$, by a polynomial of degree $p$ in one of two forms \citep{trefethen2000}.
The first is a Lagrange interpolant through the Chebyshev--Gauss--Lobatto (CGL) nodes
\begin{align}
\label{eqn:cgl_grid}
    \tau_n := \cos\left(\frac{\pi n}{p}\right), \quad n \in \{0, \ldots, p\},
\end{align}
parameterized by its nodal values $f(\tau_n)$.
The second is a B\'ezier curve
\begin{align}
\label{eqn:bezier_approx}
    f(\tau) = \sum_{b=0}^{p} B_b^p(\tau)\,\bar{f}_b,
\end{align}
parameterized by its control points $\bar{f}_b$, where $B_b^p$ are the Bernstein basis polynomials on $[-1, 1]$.
Either way a trajectory becomes a finite vector of coefficients.
The B\'ezier form has a property that the nodal form lacks, and which we use to enforce wrench constraints for all time.
\begin{lem}[B\'ezier control points certify convex constraints]
\label{lem:bezier_convex_contraints}
Let $\W \subseteq \RR^n$ be convex and let $\lambda$ be a B\'ezier curve with control points $\bar{\lambda}_0, \ldots, \bar{\lambda}_p$.
If $\bar{\lambda}_b \in \W$ for all $b$, then $\lambda(\tau) \in \W$ for all $\tau \in [-1, 1]$.
\end{lem}
\begin{proof}
The Bernstein basis is nonnegative and sums to one, so $\lambda(\tau)$ is a convex combination of the control points.
\end{proof}

    \section{Transcribing the OCP into a Gradient Flow}
\label{sec:gradient_flow}

This section transcribes Problem \ref{prob:contact_sequence_ocp} into a nonlinear program and derives the pseudo-time Gradient Flow (GF) that \methodname{} integrates to solve it.

\subsection{Transcription}
\label{subsec:transcription}

On the mode interval $I_i = [t_i, t_{i+1}]$, we introduce the reference time
\begin{equation}
\label{eqn:reference_time}
    \tau = \frac{2(t - t_i)}{t_{i+1} - t_i} \in [-1, 1],
\end{equation}
and represent each decision function on $I_i$ by a polynomial in $\tau$ using the forms of Section \ref{subsec:interpolating_polynomials}.
The state and control are Lagrange interpolants on the CGL nodes \eqref{eqn:cgl_grid}, so their coefficients are the nodal values $q(\tau_n), v(\tau_n), u(\tau_n)$, and time derivatives are obtained by the CGL differentiation matrix \citep{trefethen2000}.
% $D$ \citep{trefethen}, $\dot{v}(\tau_n) = \frac{2}{t_{i+1} - t_i}\sum_m D_{nm} v(\tau_m)$.
The contact wrench is a B\'ezier curve \eqref{eqn:bezier_approx} with control points $\bar{\lambda}_{i,0}, \ldots, \bar{\lambda}_{i,p}$.
Each field can use a different degree.
We call one such polynomial segment an \emph{element} and the collection of elements the \emph{mesh}; initially there is one element per mode interval, and the $h$-refinement of Section \ref{subsec:globalization} may later split an element in two.
Let $z_i$ collect the coefficients of mode $i$ and $z = (z_0, \ldots, z_{M-1}) \in \RR^{n_z}$ all decision variables.

The cost \eqref{eqn:ocp_integral} is approximated by Clenshaw--Curtis quadrature on the CGL nodes.
The constraints are transcribed in two ways.
The dynamics \eqref{eqn:mode_dae}, the bounds on $\x$ and $\uu$, and any other constraint involving the state or control are collocated at the CGL nodes.
The friction cone, which is convex in $\lambda_i$, is imposed on the B\'ezier control points $\bar{\lambda}_{i,b}$ rather than at the nodes, so by Lemma \ref{lem:bezier_convex_contraints} it holds for every $\tau \in [-1, 1]$ and not only at the collocation nodes.
The reset map at $t_i$ becomes an algebraic constraint linking the last node of element $i-1$ to the first node of element $i$, $q$ being continuous and $v$ mapped by $R_{j_{i-1}, j_i}$, and the boundary conditions are imposed on the first and last nodes of the mesh.
Constraints collocated at the nodes may be violated between them; Section \ref{sec:numerical_scheme} describes how to control this error through $h$-refinement and penalties.

Collecting the rows, we obtain the nonlinear program
\begin{problem}[Transcribed OCP]
\label{prob:nlp}
\begin{align}
    \min_{z \in \RR^{n_z}} \quad & J_d(z) \\
    \sut \quad & g_i(z) = 0, \quad i \in \E_d, \\
    & g_i(z) \leq 0, \quad i \in \Ii_d,
\end{align}
\end{problem}
\noindent where $\E_d$ collects the collocated dynamics, the reset maps and the boundary conditions, $\Ii_d$ collects the friction-cone rows on the control points and the collocated bounds, and $m_d := |\E_d| + |\Ii_d|$.

\subsection{Gradient Flow}
\label{subsec:gradient_flow}

We now define a flow in a pseudo-time $s$ whose equilibria are the KKT points of Problem \ref{prob:nlp}, built on the Augmented Lagrangian of Powell, Hestenes and Rockafellar (PHR) \citep{hestenes1969, powell1969, rockafellar1973}.

\begin{defn}[PHR augmented Lagrangian]
\label{defn:phr_al}
For each constraint $i \in \E_d \cup \Ii_d$, let $\gamma_i$ be its multiplier and $\mu_i > 0$ its penalty.
The \emph{effective residual} and \emph{effective multiplier} are
\begin{align}
\label{eqn:effective_constraint_residual}
c_i^{\mathrm{eff}}(z, \gamma) &:= \begin{cases}
g_i(z), & i \in \E_d, \\
\max\!\left(g_i(z), -\tfrac{\gamma_i}{\mu_i}\right), & i \in \Ii_d,
\end{cases} \\
\label{eqn:effective_lagrange_multiplier}
w_i(z, \gamma) &:= \gamma_i + \mu_i\, c_i^{\mathrm{eff}}(z, \gamma),
\end{align}
respectively, and the PHR augmented Lagrangian is
\begin{align}
\label{eq:phr_aug_lagrangian}
L_A(z, \gamma; \mu) := J_d(z) + \sum_{i=1}^{m_d} \frac{w_i(z, \gamma)^2 - \gamma_i^2}{2\mu_i}.
\end{align}
\end{defn}
\noindent For an equality constraint the $i$th summand equals $\gamma_i g_i + \frac{\mu_i}{2} g_i^2$, and for an inequality constraint $w_i = [\gamma_i + \mu_i g_i]_+$, so $L_A$ is continuously differentiable with
\begin{align}
\label{eqn:phr_gradients}
\nabla_z L_A &= \nabla J_d(z) + \sum_{i=1}^{m_d} w_i(z, \gamma)\, \nabla g_i(z),\\
\nabla_\gamma L_A &= c^{\mathrm{eff}}(z, \gamma).
\end{align}

\begin{defn}[Two-time-scale flow]
\label{defn:two_time_scale_flow}
Let $z : [0, s_{\max}] \to \RR^{n_z}$ be a curve in pseudo-time $s$, and partition $[0, s_{\max}]$ into windows $[s_k, s_{k+1})$.
On each window $z$ follows the primal flow at frozen multipliers $\gamma = \gamma(s_k)$,
\begin{align}
    \label{eqn:frozen_flow}
    \dot{z}(s) = -\nabla_z L_A(z(s), \gamma; \mu),
\end{align}
and at $s_{k+1}$ the multipliers jump by the damped update
\begin{align}
\label{eqn:multiplier_update}
\gamma^+ := \gamma + \tau_\gamma\, \mu \odot c^{\mathrm{eff}}(z(s_{k+1}), \gamma), \qquad \tau_\gamma \in (0, 1].
\end{align}
\end{defn}

\begin{prop}[Equilibria and monotonicity]
\label{prop:flow_properties}
Fix $\mu \in \RR^{m_d}_{>0}$.
\begin{enumerate}
    \item[(i)] If $\nabla_z L_A(z, \gamma; \mu) = 0$ and $\gamma^+ = \gamma$, then $(z, \gamma)$ is a KKT point of Problem \ref{prob:nlp}.
    \item[(ii)] Along any solution of \eqref{eqn:frozen_flow} on a window,
    \begin{align}
        \label{eqn:lyapunov_decrease}
        \frac{d}{ds} L_A(z(s), \gamma; \mu) = -\|\nabla_z L_A(z(s), \gamma; \mu)\|^2 \leq 0,
    \end{align}
    with equality if and only if $\nabla_z L_A = 0$, so $L_A(\cdot, \gamma; \mu)$ is non-increasing on every window and the equilibria of \eqref{eqn:frozen_flow} are the stationary points of $\min_z L_A(z, \gamma; \mu)$.
\end{enumerate}
\end{prop}
\begin{proof}
(i) Since $\tau_\gamma \mu_i > 0$, $\gamma^+ = \gamma$ gives $c^{\mathrm{eff}} = 0$.
For $i \in \E_d$ this is $g_i = 0$; for $i \in \Ii_d$, $\max(g_i, -\gamma_i/\mu_i) = 0$ gives $g_i \leq 0$, $\gamma_i \geq 0$ and $\gamma_i g_i = 0$.
Then $w_i = \gamma_i$ for all $i$, and \eqref{eqn:phr_gradients} turns $\nabla_z L_A = 0$ into $\nabla J_d + \sum_i \gamma_i \nabla g_i = 0$, which completes the KKT conditions.
(ii) Apply the chain rule along \eqref{eqn:frozen_flow}: $\frac{d}{ds} L_A = \nabla_z L_A^T \dot{z}$.
\end{proof}

\noindent The multiplier update \eqref{eqn:multiplier_update} ascends $L_A$, so monotonicity is a property of the primal flow between multiplier updates and not of the composed map $\gamma \mapsto \gamma^+$.
This is the property the integrator of Section \ref{sec:numerical_scheme} is designed to preserve.
    \section{Numerical scheme}
\label{sec:numerical_scheme}

Section \ref{sec:gradient_flow} reduced trajectory optimization to integrating the flow \eqref{eqn:frozen_flow}.
Doing so with an off-the-shelf integrator fails for three reasons.
First, the flow is stiff when the penalties are large \citep{hairer1996}.
Second, its Jacobian jumps whenever an inequality constraint enters or leaves the active set $\A(z, \gamma) := \E_d \cup \{i \in \Ii_d : w_i(z, \gamma) > 0\}$, which violates the Lipschitz assumption behind the error control of implicit integrators.
Third, the point of the flow is the monotonicity of Proposition \ref{prop:flow_properties}(ii), and a discretization that does not preserve it forfeits the robustness that motivated the approach.
Section \ref{subsec:linearly_implicit_step} builds a step that handles stiffness and the jumps, Section \ref{subsec:acceptance} controls that step so that monotonicity survives discretization, Section \ref{subsec:window} discretizes the outer time scale of the multipliers, Section \ref{subsec:acceptance_criterion} decides when the discrete solution is good enough, and Section \ref{subsec:globalization} handles initialization and mesh adaptation.
Algorithm \ref{alg:skeleton} shows how these components fit together.

\begin{algorithm}[t]
\caption{\methodname{} (schematic)}
\label{alg:skeleton}
\small
\begin{algorithmic}[1]
\Require initial guess $z_0$, tolerance $\mathrm{tol}$, penalty seed $\mu_0$
\State $z \leftarrow \textsc{FeasibilityFlow}(z_0)$ \Comment{Sec.~\ref{subsec:globalization}}
\For{level $= 0, \ldots, n_{\mathrm{lev}}$} \Comment{one mesh}
  \State $\gamma \leftarrow 0$; $\mu \leftarrow \mu_0$
  \Repeat \Comment{one window: $\gamma, \mu$ frozen}
    \Repeat \Comment{one step}
      \State $v \leftarrow \textsc{LookAhead}(z)$ \Comment{Sec.~\ref{subsec:acceptance}}
      \State solve \eqref{eqn:li_step} at $v$ for $\Delta z$ \Comment{Def.~\ref{defn:li_step}}
      \State $(z, h) \leftarrow \textsc{AcceptAndControl}(v, \Delta z, h)$ \Comment{Lemma~\ref{lem:monotone_sequence}, \eqref{eqn:step_control}}
    \Until{window budget or stationarity signal}
    \State $(\gamma, \mu) \leftarrow \textsc{UpdateMultipliers}(z, \gamma, \mu)$ \Comment{\eqref{eqn:multiplier_update}, \eqref{eqn:penalty_update}}
  \Until{$\textsc{LevelStop}(z)$} \Comment{Def.~\ref{defn:criterion}}
  \If{$G(z) \leq \mathrm{tol}$ and $F(z) \leq \mathrm{tol}$} \Return $z$ \EndIf
  \State $z \leftarrow \textsc{Refine}(z)$ \Comment{Sec.~\ref{subsec:globalization}}
\EndFor
\end{algorithmic}
\end{algorithm}

\subsection{Linearly implicit step}
\label{subsec:linearly_implicit_step}

Stiffness calls for an implicit method, but a fully implicit step of \eqref{eqn:frozen_flow} requires solving a nonlinear system at every step, and the exact Jacobian $\nabla^2_{zz} L_A$ contains second derivatives of the dynamics, which are expensive, and the rank-one terms $\mu_i \nabla g_i \nabla g_i^T$ that appear and disappear with the active set.
We therefore take a \emph{linearly implicit} step: one linear solve per step, against a Jacobian model that is cheap, positive semidefinite, and already contains the terms the exact Jacobian is about to gain.
A \emph{step} in the rest of the section means one such update.

\begin{defn}[Gauss--Newton model on an anticipatory active set]
\label{defn:gn_metric}
For a margin $\varepsilon_\A \geq 0$, let $\A_{\varepsilon}(z, \gamma) := \E_d \cup \{ i \in \Ii_d : \gamma_i + \mu_i(g_i(z) + \varepsilon_\A) > 0 \} \supseteq \A(z, \gamma)$ and define
\begin{align}
\label{eqn:gn_metric}
H(z, \gamma) := H_0(z) + J_{\A_\varepsilon}(z)^T \operatorname{diag}(\mu_{\A_\varepsilon}) J_{\A_\varepsilon}(z),
\end{align}
where $H_0(z) \succeq 0$ is the Gauss--Newton Hessian of $J_d$ and $J_{\A_\varepsilon}$ stacks the gradients $\nabla g_i^T$ for $i \in \A_\varepsilon$.
\end{defn}
$H$ is the Gauss--Newton approximation of $\nabla^2_{zz} L_A$, obtained from \eqref{eqn:phr_gradients} by dropping the curvature terms $w_i \nabla^2 g_i$.
The margin $\varepsilon_\A$ enlarges the active set by the constraints within $\varepsilon_\A$ of switching, so that the model already carries the rank-one term the exact Jacobian will gain within the step.

\begin{defn}[Linearly implicit step]
\label{defn:li_step}
Given a linearization point $v$, frozen multipliers $\gamma$, a step $h > 0$ and $\theta \in [\tfrac12, 1]$, the $\theta$-method applied to \eqref{eqn:frozen_flow} with Jacobian model $-H(v, \gamma)$ is
\begin{align}
\label{eqn:li_step}
\big( I + h\theta H(v, \gamma) \big) \, \Delta z = -h \nabla_z L_A(v, \gamma),
\end{align}
followed by $z^+ = v + \alpha \Delta z$ with $\alpha \in (0, 1]$ chosen in Section \ref{subsec:acceptance}.
\methodname{} uses $\theta = 1$.
\end{defn}
Dividing \eqref{eqn:li_step} by $h\theta$ shows that $\Delta z$ solves $(H + \tfrac{1}{h\theta} I)\Delta z = -\tfrac{1}{\theta}\nabla_z L_A$: a Levenberg--Marquardt step with damping $1/(h\theta)$, which is a gradient step as $h \to 0$ and the Gauss--Newton step as $h \to \infty$.
The pseudo-time step therefore doubles as a trust-region radius, and the question of how large a step to take becomes the question of how far to trust the model $H$.
The next lemma provides the quantity that answers it.

\begin{lem}[Descent and predicted decrease]
\label{lem:li_step}
For every $h > 0$ and $\theta \in [\tfrac12, 1]$, the step $\Delta z$ of \eqref{eqn:li_step} exists, is unique, and satisfies $\nabla_z L_A^T \Delta z < 0$ whenever $\nabla_z L_A \neq 0$.
The quadratic model $m(\alpha) := L_A(v) + \alpha \nabla_z L_A^T \Delta z + \tfrac{\alpha^2}{2}\Delta z^T H \Delta z$ predicts the decrease
\begin{align}
\label{eqn:pred}
\mathrm{pred}(\alpha) := m(0) - m(\alpha) > 0.
\end{align}
\end{lem}
\begin{proof}
$H$ is a sum of Gram matrices, so $I + h\theta H \succeq I$ is invertible and $\nabla_z L_A^T \Delta z = -h \nabla_z L_A^T (I + h\theta H)^{-1} \nabla_z L_A < 0$.
From \eqref{eqn:li_step}, $\Delta z^T H \Delta z = -(h \nabla_z L_A^T \Delta z + \|\Delta z\|^2)/(h\theta)$; substituting into $m(\alpha)$ gives $\mathrm{pred}(\alpha) = \alpha ( 1 - \frac{\alpha}{2\theta} ) ( -\nabla_z L_A^T \Delta z ) + \frac{\alpha^2 \|\Delta z\|^2}{2 h \theta}$, which is positive since $\alpha \leq 1 \leq 2\theta$.
\end{proof}

\subsection{Acceptance and step control}
\label{subsec:acceptance}

A stiff integrator chooses $h$ from a local truncation-error estimate, which presumes a smooth vector field and misjudges the step whenever the active set changes.
We instead choose $h$ the way an optimizer would: a step is accepted only if $L_A$ decreases, which makes monotonicity a property of the discrete iterates by construction, and $h$ is grown or shrunk according to how well the model of Lemma \ref{lem:li_step} predicted that decrease.
Both tests use only values of $L_A$, so neither sees the Jacobian jumps.

The linearization point $v_k$ is either the last accepted iterate $z_k$ or an extrapolation of $z_k$ and $z_{k-1}$, which is discarded whenever $L_A(v_k, \gamma) > L_A(z_k, \gamma)$, so that $L_A(v_k, \gamma) \leq L_A(z_k, \gamma)$ in either case.
The trial $z^+(\alpha) = v_k + \alpha \Delta z$ is accepted, and becomes $z_{k+1}$, if and only if $L_A(z^+(\alpha), \gamma) < L_A(v_k, \gamma)$, starting at $\alpha = 1$ and backtracking on rejection; an accepted step advances pseudo-time by $\alpha h$.
Chaining the two inequalities gives the discrete counterpart of Proposition \ref{prop:flow_properties}(ii):
\begin{lem}[Monotone accepted sequence]
\label{lem:monotone_sequence}
The accepted iterates satisfy $L_A(z_{k+1}, \gamma) < L_A(z_k, \gamma)$ for every $k$, with or without look-ahead.
\end{lem}

The step $h$ is controlled by the ratio $\varrho := \big( L_A(v) - L_A(z^+(\alpha)) \big) / \mathrm{pred}(\alpha)$ of actual to predicted decrease, as in a trust-region method \citep[Ch.~4]{nocedal2006}:
\begin{align}
\label{eqn:step_control}
h \leftarrow \begin{cases}
\min(g_+ h, h_{\max}), & \varrho > \eta_2 \text{ and } \alpha = 1, \\
\max(g_- h, h_{\min}), & \varrho < \eta_1, \\
h/2, & \alpha < 1, \\
h, & \text{otherwise,}
\end{cases}
\end{align}
with $0 < \eta_1 < \eta_2 < 1$ and $g_+ > 1 > g_- > 0$.
A large $\varrho$ means the Gauss--Newton model is accurate and $h$ grows toward the Gauss--Newton step; a small $\varrho$ shrinks it toward a gradient step.
The step-size controller is thus deciding how far to trust $H$, not how accurately to follow the flow in pseudo-time, which is why it can be insensitive to the jumps that defeat the truncation-error estimators of stiff integrators \citep{hairer1996}.

\subsection{Windows, multipliers and penalties}
\label{subsec:window}

The flow of Definition \ref{defn:two_time_scale_flow} holds $\gamma$ fixed while $z$ descends and updates $\gamma$ only once the inner problem is near stationarity.
We call the run of accepted steps between two multiplier updates a \emph{window}.
A window ends when its accumulated pseudo-time reaches $\Delta s$, or earlier when $\|\nabla_z L_A\|$ falls below $\epsilon_g$ or the decrease of $L_A$ stalls.
At the window boundary the feasibility criterion of Section \ref{subsec:acceptance_criterion} is tested; if it fails, the multipliers are updated by \eqref{eqn:multiplier_update} with $\tau_\gamma = \tfrac12$, and any penalty whose constraints are not converging is increased.

The penalty update acts on groups of constraints rather than on all of them at once.
Let $\nu_i := |g_i|$ for $i \in \E_d$ and $\nu_i := [g_i]_+$ for $i \in \Ii_d$ denote the violation of constraint $i$.
The rows of one named constraint of Problem \ref{prob:contact_sequence_ocp} form a \emph{term} $t$ with index set $\R_t$, violation $\mathrm{cv}_t := \max_{i \in \R_t} \nu_i$ and a shared penalty $\mu_t$, and a term whose violation has not contracted by $\kappa_\mu \in (0,1)$ since the previous boundary has its penalty increased:
\begin{align}
\label{eqn:penalty_update}
\mathrm{cv}_t > \kappa_\mu \, \mathrm{cv}_t^{\mathrm{prev}} \implies \mu_t \leftarrow \min(\varrho_\mu \mu_t, \mu_{\max}),
\end{align}
where $\varrho_\mu > 1 $.
The reason for per-term penalties is a degeneracy of the uniform choice.
Suppose the window has converged, $\nabla_z L_A = 0$, and apply \eqref{eqn:multiplier_update} at that $z$.
By \eqref{eqn:phr_gradients} the gradient becomes $\tau_\gamma \sum_i \mu_i\, c_i^{\mathrm{eff}} \nabla g_i$.
If $\mu_i \equiv \mu$ this is $\tau_\gamma \mu$ times the gradient of the infeasibility $\tfrac12 \sum_i (c_i^{\mathrm{eff}})^2$, so at an infeasible stationary point of the infeasibility the update does nothing and raising $\mu$ does not help.
Distinct $\mu_t$ weight the sum differently and generically move the iterate off such a point.

\subsection{Feasibility criterion}
\label{subsec:acceptance_criterion}

The transcription of Section \ref{subsec:transcription} enforces the path constraints of Problem \ref{prob:contact_sequence_ocp} at the CGL nodes.
Deciding when the discrete solution is good enough therefore involves two errors of different origin: the violation at the nodes, which the flow drives down, and the violation between nodes, which only refining the mesh can reduce.
The criterion measures both, so that the scheme can tell whether to keep flowing or to refine.
\begin{defn}[Node gate, fine-grid violation and criterion]
\label{defn:criterion}
Let
\begin{align}
\label{eqn:criterion}
G(z) &:= \max_i \nu_i(z), \\
F(z) &:= \max_e F_e(z), \\
F_e(z) &:= \max_{t_k \in [t_e, t_{e+1}]} \, \max_r \, \nu_r(t_k),
\end{align}
where $t_k = k \Delta t$ is a fine grid of physical times and $r$ runs over the path constraints of element $e$ evaluated on its polynomials.
$G$ is the \emph{node gate}, $F$ the \emph{fine-grid violation}, and the \emph{criterion} is $G(z) \leq \mathrm{tol}$ and $F(z) \leq \mathrm{tol}$.
\end{defn}
A \emph{level} is the sequence of windows run on one mesh.
Each time a window ends, the criterion is evaluated.
If it holds, the level ends and the solve is complete.
If $G \leq \mathrm{tol}$ but $F > \mathrm{tol}$, the nodes have converged and the remaining violation lies between them, which the flow cannot reduce; the level ends and the outer loop refines the mesh (Section \ref{subsec:globalization}).
Otherwise the flow continues into the next window.

\subsection{Outer loop}
\label{subsec:globalization}

Two things remain outside the inner iteration.
It converges to a stationary point of $\min_z L_A$ in whose basin it starts, so a poor initial guess must first be brought into a useful basin; and it can only reduce the nodal violation $G$, so when a level ends with $F > \mathrm{tol}$ the mesh must change.

\paragraph{Feasibility flow}
Before the objective and the multipliers enter, the flow \eqref{eqn:frozen_flow} is run on the special case of $L_A$ with $\gamma = 0$, a common penalty $\omega$ and $J_d$ dropped, so that $L_A$ reduces to $\Phi(z) := \tfrac{\omega}{2} \sum_i \nu_i(z)^2$.
The step \eqref{eqn:li_step} is then the Levenberg--Marquardt iteration for $\min_z \Phi(z)$ \citep[Ch.~10]{nocedal2006}, and this stage removes the bulk of the infeasibility of the initial guess, which in our experiments is a linear interpolation between boundary poses.

\paragraph{$h$-refinement}
When a level ends with $F > \mathrm{tol}$, each element receives the indicator $\eta_e := \max(\text{nodal violation on } e,\ F_e)$, and D\"orfler marking \citep{dorfler1996} bisects the elements with the largest indicators at the time of their worst violation \citep{darby2011hp}.
The children keep the parent's degrees, and the iterate is transferred to the new mesh by re-interpolating the nodal fields at the children's CGL nodes and splitting each B\'ezier wrench by de Casteljau subdivision \citep{farin2002}.
Interpolation at $p + 1$ nodes reproduces a degree-$p$ polynomial exactly, and de Casteljau subdivision returns the same curve on each half with control points that are convex combinations of the parent's, so Lemma \ref{lem:bezier_convex_contraints} applies to the children.
% \begin{lem}[Exact prolongation]
% \label{lem:prolongation}
% The prolonged iterate is unchanged as a function of time; every constraint value is preserved, and the children's wrench control points lie in the friction cone whenever the parent's do.
% \end{lem}
Because the constraint set changes, the multipliers are reset and the penalties reseeded; the refined iterate is nearly feasible, so the flow restarts with a large step.
The solve ends when the criterion holds, when $n_{\mathrm{lev}}$ refinements are exhausted, or when refinement stops reducing $G$, which indicates a wrong basin rather than a discretization error.
    \section{Simulation experiments}
\label{sec:simulation_experiments}

This section compares \methodname{} against the state of the art methods Crocoddyl \citep{croco} and Aligator \citep{aligator}.
All methods solve instances of Problem \ref{prob:contact_sequence_ocp} for different robot morphologies and under different constraints.
Subsection \ref{subsec:experimental_setup} describes the experimental setup that will allow us to fairly compare the different methods.
Subsection~\ref{subsec:scalability_to_convergence} compares optimization times and the success of the resulting trajectories under simulated tracking.
Subsection \ref{subsec:scalability_time_budget} illustrates how each method fares when given a time budget to solve the OCP.

\subsection{Experimental setup}
\label{subsec:experimental_setup}
To fairly compare the different methods, we define a set of benchmark problems and criteria to evaluate the solution.
All experiments are run in a machine with a 2x AMD EPYC 7742 64-Core Processor and for each experiment each method is allocated 12 physical cores.

\subsubsection{Benchmark Problems}
The results presented here comprise ten benchmark problems across two tasks: walking over uneven terrain and jumping over a fence (Figure~\ref{fig:timelapse_scenarios}).
Each task is evaluated on five robot platforms: a planar hopper($n_j=2$), Unitree's Go2 quadruped ($n_j=12$), H1 biped ($n_j=19$), G1 humanoid ($n_j=29$), and PAL Robotics' Talos humanoid ($n_j=32$).
All methods solve the same optimization problem to the extend their formulations allow: same cost function, constraint and initial guess.
 
\subsubsection{Solution quality criteria}
Each optimization method produces optimal trajectories for the states and torques at the collocation nodes.
Small constraint residuals at the collocation nodes do not ensure that a tracking controller can execute the resulting motion and reach the prescribed terminal state.
We therefore evaluate the references through closed-loop simulation and assess terminal goal accuracy separately from optimization feasibility.
A whole-body inverse-dynamics controller tracks each reference in MuJoCo, using the same controller configuration for all optimization methods on a given task and robot.
We refer to the state obtained by numerically integrating the robot dynamics under the controller's torques and the simulator's contact forces as the simulated rollout.

For the solve-time comparisons, success is determined by the rollout's error relative to the problem's prescribed terminal rest pose and foot placements.
The position errors and residual speeds must be strictly below the six thresholds in Table~\ref{tab:tracking_success_thresholds}, with identical numerical thresholds for every robot and method.

\begin{table}[t]
    \centering
    \footnotesize
    \setlength{\tabcolsep}{4pt}
    \begin{tabular}{@{}llr@{}}
        \toprule
        \textbf{Quantity} & \textbf{Error metric} & \textbf{Threshold} \\
        \midrule
        Base position
        & Euclidean position error
        & $0.050$ m \\

        Base orientation
        & Max. absolute wrapped RPY error
        & $0.10$ rad \\

        Joint position
        & Max. absolute actuated-joint error
        & $0.20$ rad \\

        Foot position
        & Max. Euclidean foot error
        & $0.040$ m \\

        Base linear speed
        & Linear velocity norm
        & $0.10$ m/s \\

        Joint speed
        & Max. absolute actuated-joint velocity
        & $0.10$ rad/s \\
        \bottomrule
    \end{tabular}
    \caption{Terminal goal thresholds shared by all robots and methods.
    Each metric must be strictly below its threshold at the final time of the rollout.}
    \label{tab:tracking_success_thresholds}
\end{table}

\subsection{Scalability to convergence}
\label{subsec:scalability_to_convergence}
We compare \methodname{}, Crocoddyl and Aligator on the Fence and Uneven terrain benchmark.
Figure~\ref{fig:timelapse_scenarios} shows the rollout solution produced by \methodname{} for all the scenarios and all the robots.
The goal of this subsection is to compare the solve time of the different methods for solutions whose rollout satisfy the success criteria.

Table~\ref{tab:solve_time_results} reports the optimization time of the reference trajectories whose simulated rollouts satisfy the terminal tracking criteria in Table~\ref{tab:tracking_success_thresholds}, with the five robot platforms across the columns and the two tasks down the rows.
Crosses indicate that the corresponding rollout does not satisfy these criteria.
\revise{Thus, the table assesses the computational cost of producing successfully tracked references, rather than convergence to a common optimizer stopping tolerance.}{Each entry is the solver time at which a method produced its final iterate within a 14400~s budget.
That iterate is the reference tracked in simulation. 
All three methods stop early once the iterate satisfies the shared feasibility test (maximum node and fine-grid constraint violation below $10^{-3}$) or their own convergence test.}
Across the ten displayed cases, \methodname{} satisfies the tracking criteria in all ten, compared with seven for Crocoddyl and eight for Aligator.

\begin{table}[!t]
    \centering
    \footnotesize
    \setlength{\tabcolsep}{4pt}
    \begin{tabular}{@{}lrrrrr@{}}
        \toprule
        Solve time [s] & Hopper & Go2 & H1 & G1 & Talos \\
        $n_j$ & 2 & 12 & 19 & 29 & 32 \\
        \midrule
        \multicolumn{6}{@{}l}{\textit{Uneven terrain}} \\
        \textcolor{methodleap}{LEAP} & $\mathbf{20.9}$ & $79.8$ & $\mathbf{13693}$ & $\mathbf{1218}$ & $\mathbf{4096}$ \\
        \textcolor{methodcrocoddyl}{Crocoddyl} & $143$ & $\mathbf{78.1}$ & \revise{$14397$}{$14400^\dagger$} & $1425$ & \textcolor{methodcrocoddyl}{$\times$} \\
        \textcolor{methodaligator}{Aligator} & $1201$ & $487$ & \revise{$14397$}{$14400^\dagger$} & $2121$ & \textcolor{methodaligator}{$\times$} \\
        \midrule
        \multicolumn{6}{@{}l}{\textit{Fence jumping}} \\
        \textcolor{methodleap}{LEAP} & $\mathbf{5.2}$ & $\mathbf{71.6}$ & $176$ & $552$ & $\mathbf{658}$ \\
        \textcolor{methodcrocoddyl}{Crocoddyl} & \textcolor{methodcrocoddyl}{$\times$} & $557$ & $\mathbf{34.4}$ & $1867$ & \textcolor{methodcrocoddyl}{$\times$} \\
        \textcolor{methodaligator}{Aligator} & \textcolor{methodaligator}{$\times$} & $1200$ & \revise{$14399$}{$14400^\dagger$} & $\mathbf{301}$ & $8443$ \\
        \bottomrule
    \end{tabular}
    \caption{Optimization time of the reference trajectories on both benchmarks.
    Each entry is the solve time in seconds of the reference whose simulated rollout satisfies the terminal-goal criteria of Table~\ref{tab:tracking_success_thresholds}.
    $\times$ marks a method whose rollout does not\added{, and $\dagger$ a run that reached the budget without meeting the feasibility test}.
    \textbf{Bold} marks the fastest successful method for each problem, times above
    100~s are rounded to the second and method colors match
    Figure~\ref{fig:budget_violation}.
    $n_j$ is the actuated joint count.}
    \label{tab:solve_time_results}
\end{table}

Table~\ref{tab:terminal_goal_results} summarizes each rollout by its largest normalized terminal error in the terminal state.
\methodname{} attains the smallest normalized error on nine of the ten problems. 
The H1 fence-jumping experiment is the only instance where a comparison method obtains better terminal accuracy.

% Generated by plotting_paper_code/terminal_goal_table/make_table.py.
% Do not edit by hand; re-run the generator and paper/sync_figures.sh.
\definecolor{methodleap}{rgb}{0.000,0.450,0.700}
\definecolor{methodaligator}{rgb}{0.900,0.620,0.000}
\definecolor{methodcrocoddyl}{rgb}{0.800,0.400,0.650}
\begin{table}[!t]
    \centering
    \footnotesize
    \setlength{\tabcolsep}{5pt}
    \begin{tabular}{@{}llrr@{}}
        \toprule
        & & \multicolumn{2}{c}{$\max_i e_i/\bar{e}_i$} \\
        \cmidrule(lr){3-4}
        Robot & Method & Uneven terrain & Fence jumping \\
        \midrule
        \multirow{3}{*}{Hopper (2)} & \textcolor{methodleap}{LEAP} & $\mathbf{0.12}$ & $\mathbf{0.89}$ \\
         & \textcolor{methodcrocoddyl}{Crocoddyl} & $0.14$ & $147$ \\
         & \textcolor{methodaligator}{Aligator} & $0.17$ & $14$ \\
        \cmidrule(lr){1-4}
        \multirow{3}{*}{Go2 (12)} & \textcolor{methodleap}{LEAP} & $\mathbf{0.23}$ & $\mathbf{0.038}$ \\
         & \textcolor{methodcrocoddyl}{Crocoddyl} & $0.32$ & $0.31$ \\
         & \textcolor{methodaligator}{Aligator} & $0.36$ & $0.76$ \\
        \cmidrule(lr){1-4}
        \multirow{3}{*}{H1 (19)} & \textcolor{methodleap}{LEAP} & $\mathbf{0.028}$ & $0.37$ \\
         & \textcolor{methodcrocoddyl}{Crocoddyl} & $0.13$ & $0.21$ \\
         & \textcolor{methodaligator}{Aligator} & $0.18$ & $\mathbf{0.16}$ \\
        \cmidrule(lr){1-4}
        \multirow{3}{*}{G1 (29)} & \textcolor{methodleap}{LEAP} & $\mathbf{0.016}$ & $\mathbf{0.0014}$ \\
         & \textcolor{methodcrocoddyl}{Crocoddyl} & $0.047$ & $0.24$ \\
         & \textcolor{methodaligator}{Aligator} & $0.048$ & $0.34$ \\
        \cmidrule(lr){1-4}
        \multirow{3}{*}{Talos (32)} & \textcolor{methodleap}{LEAP} & $\mathbf{0.13}$ & $\mathbf{0.0023}$ \\
         & \textcolor{methodcrocoddyl}{Crocoddyl} & $633$ & $678$ \\
         & \textcolor{methodaligator}{Aligator} & $569$ & $0.0053$ \\
        \bottomrule
    \end{tabular}
    \caption{Terminal goal accuracy of the simulated rollouts on both benchmarks.
    Each entry is $\max_i e_i/\bar{e}_i$, the largest of the six errors of
    Table~\ref{tab:tracking_success_thresholds} divided by its threshold evaluated at the terminal time.
    A value below one means the rollout reaches the goal.
    \textbf{Bold} marks the smallest value for each problem, and method colors match
    Table~\ref{tab:solve_time_results} and Figure~\ref{fig:budget_violation}.
    Parentheses give the actuated joint count, listed as terrain/fence for the Hopper,
    which the two tasks pose differently.
    It can be seen that for $9/10$ cases \methodname{} \revise{achieves better constraint satisfaction}{attains the smaller terminal error}.
    }
    
    \label{tab:terminal_goal_results}
\end{table}

\begin{figure}[!t]
    \centering
    \includegraphics[width=\columnwidth]{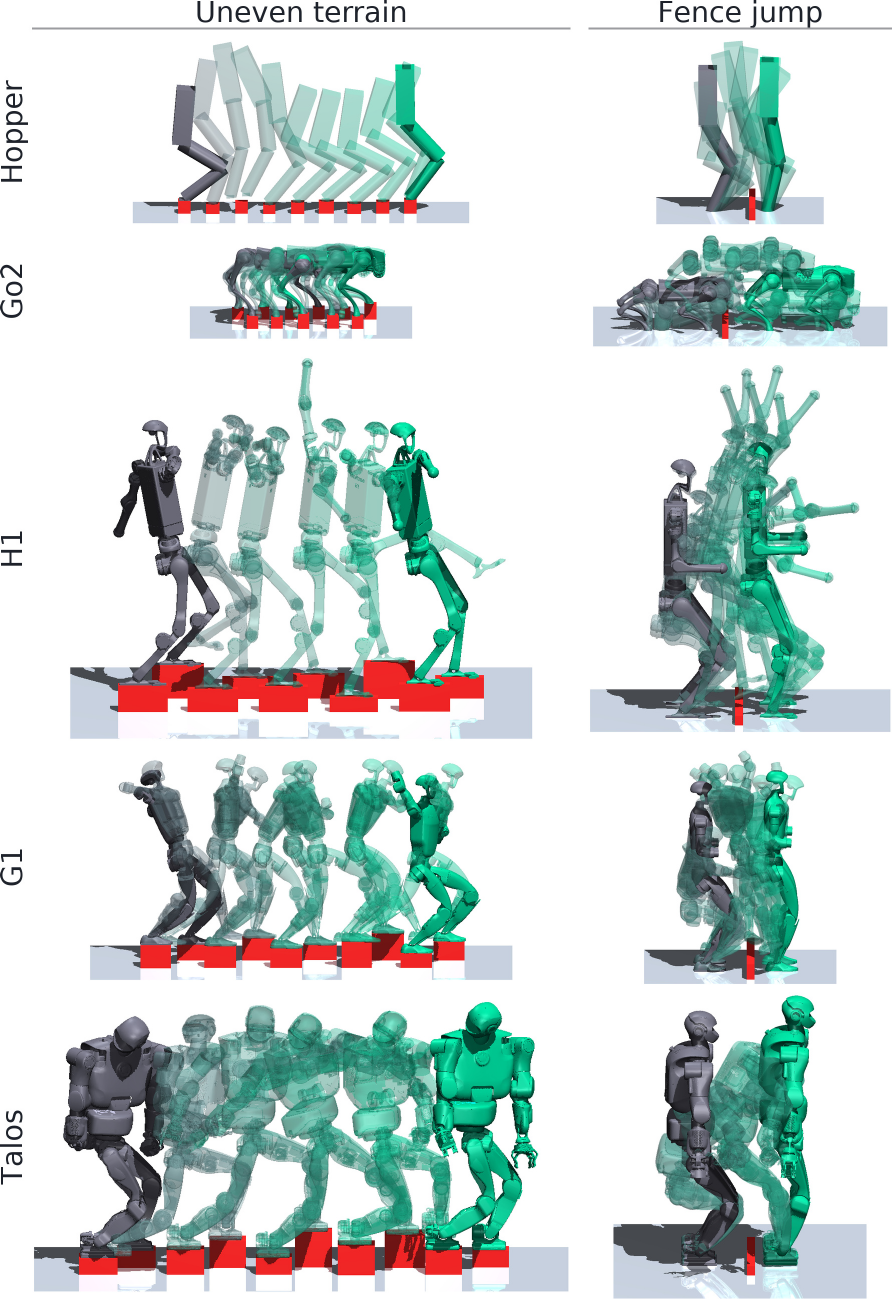}
    \caption{
    Simulated tracking of \methodname{} trajectories for both tasks, in side view: walking over uneven terrain (left) and fence jumping (right).
    Obstacles are red, and robot configurations are colored by elapsed time within that robot's plan, from the initial pose in solid gray to the final pose in solid green.
    The stepping stones are placed at distinct heights and positions making the gait non-periodic, and self-collision and collision with the fence are included in the optimization formulation.
    }
    \label{fig:timelapse_scenarios}
\end{figure}

\begin{figure}[!t]
    \centering
    \includegraphics[width=\columnwidth]{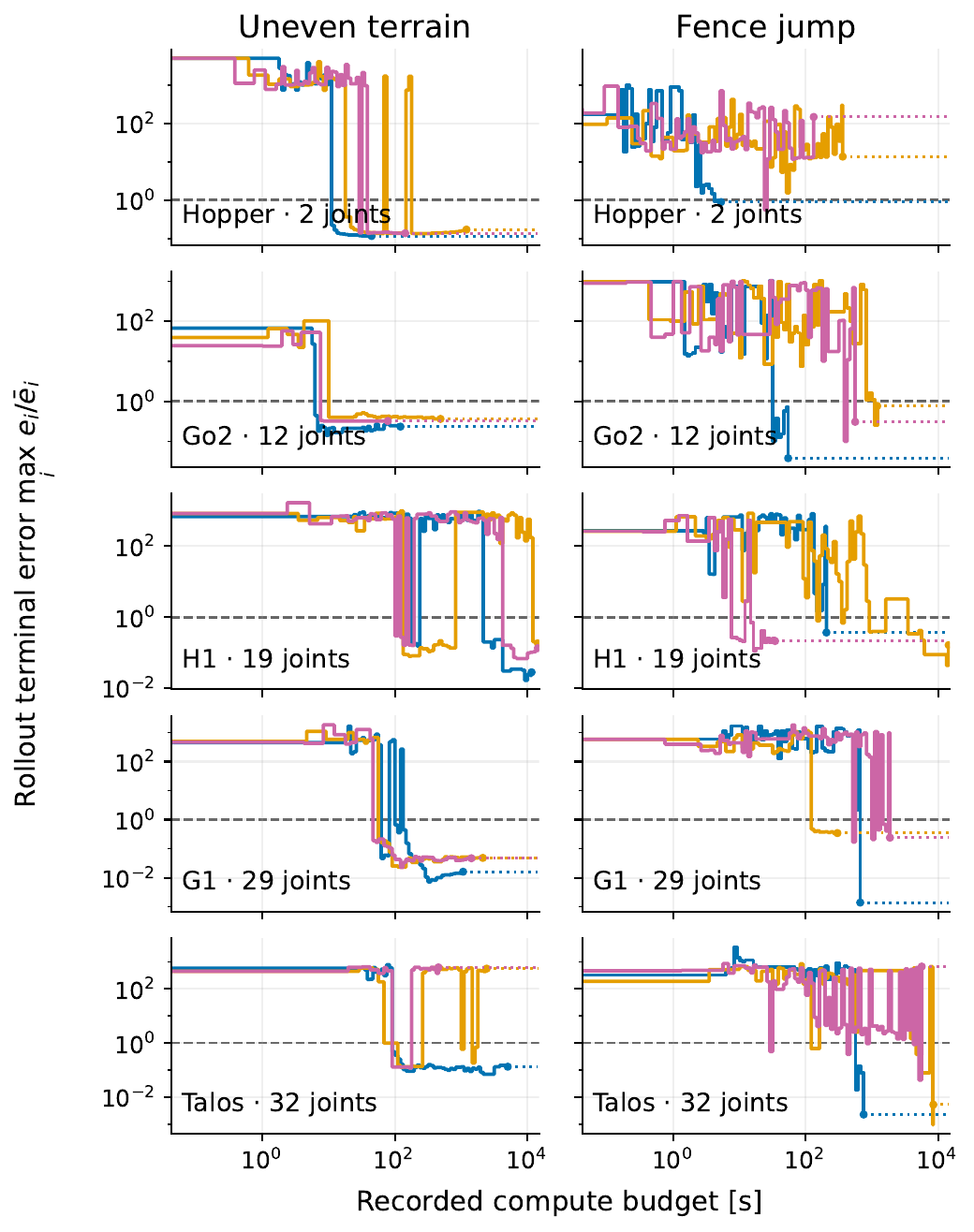}
    \caption{
    This figure shows terminal goal error $\max_i e_i/\bar{e}_i$ (Table~\ref{tab:tracking_success_thresholds}) of the rollout versus compute budget.
    The dashed line at one is the success criterion.
    At each budget value we simulate the solution with the smallest node violation recorded up to that point.
    A dot marks where the run passed the solver's convergence test or stalled, and the dotted tail carries its final iterate, the reference of Tables~\ref{tab:solve_time_results} and~\ref{tab:terminal_goal_results}, to larger budgets.
    \textcolor{methodleap}{\methodname{}} (blue), \textcolor{methodaligator}{Aligator} (orange), \textcolor{methodcrocoddyl}{Crocoddyl} (pink).
    Counting the budget from which the retained reference stays trackable, \methodname{} is earliest on seven of ten problems.
    }
    \label{fig:budget_violation}
\end{figure}

\subsection{Scalability under a time budget}
\label{subsec:scalability_time_budget}

We next examine how the quality of the simulated rollout improves with the available computation time for the different trajectory optimizers.
For each compute budget we retain, the recorded candidate with the smallest node constraint violation available, track it in simulation with the controller of Section~\ref{subsec:experimental_setup}, and report the normalized terminal error $\max_i e_i/\bar{e}_i$ of Table~\ref{tab:terminal_goal_results}.
Figure~\ref{fig:budget_violation} plots this error against the budget, with uneven terrain in the left column and fence jumping in the right column.
A curve below the dashed line at one means that the reference available at that budget is tracked to the terminal goal.
Counting the budget from which the retained reference stays trackable to the end of the run, \methodname{} reaches that point first on seven of the ten problems.

    \section{Hardware Experiments}
\label{sec:hardware_experiments}
In this section we evaluate whether a hopscotch trajectory generated by \methodname{} can be executed on a physical quadruped using a tracking controller.
The task combines forward jumps, alternating diagonal support, and a widened four-foot stance, as illustrated in Figure~\ref{fig:cover}.

\subsection{Hardware Setup}
\label{subsec:hardware_setup}
Experiments use a Unitree Go2 EDU quadruped with 12 actuated joints and a low-level torque-control interface.
Robot measurements and control commands are exchanged through the Unitree SDK~2~\citep{unitree_sdk2}.
Joint positions and the base orientation are obtained from onboard measurements.
A Vicon motion-capture system records the position of the floating base in the torso.

\subsection{Controller and System Identification}
\label{subsec:controller_sysid}

Execution in hardware follows the reference-generation and feedback-tracking approach used in Section~\ref{sec:simulation_experiments}.
To reduce the model mismatch we employ system identification ~\citep{zhang2025provably,zhang2024system} and the tracking controller is trained with domain randomization to accommodate remaining discrepancies between simulation and hardware.
The controller is inspired by Deep Tracking Control~\citep{hutterDTC} and observes joint angles, floating-base position and orientation, and base angular velocity.
\methodname{} supplies the reference motion, while the controller provides feedback corrections during execution.

\subsection{Results}
\label{subsec:results_hardware}
Table~\ref{tab:hardware_tracking} summarizes the errors for the execution shown in Figure~\ref{fig:cover}.
This execution demonstrates that a \methodname{} reference can guide the robot through the hopscotch sequence with centimeter-scale RMS position errors.

\begin{table}[!htbp]
    \centering
    \footnotesize
    \setlength{\tabcolsep}{4pt}
    \begin{tabular}{@{}lrrr@{}}
        \toprule
        \textbf{Quantity} & \textbf{RMS error}
        & \shortstack[r]{\textbf{Peak absolute}\\\textbf{error}}
        & \shortstack[r]{\textbf{Final absolute}\\\textbf{error}} \\
        \midrule
        Joint position [rad] & $0.0682$ & $0.4455$ & $0.1546$ \\
        Base roll [rad]      & $0.0249$ & $0.0572$ & $0.0435$ \\
        Base pitch [rad]     & $0.0426$ & $0.1093$ & $0.0670$ \\
        Base yaw [rad]       & $0.0113$ & $0.0309$ & $0.0107$ \\
        Base $x$ [m]         & $0.0424$ & $0.0917$ & $0.0309$ \\
        Base $y$ [m]         & $0.0347$ & $0.0769$ & $0.0093$ \\
        Base $z$ [m]         & $0.0207$ & $0.0729$ & $0.0065$ \\
        \bottomrule
    \end{tabular}
    \caption{Tracking errors for hardware execution of the hopscotch trajectory.}
    \label{tab:hardware_tracking}
\end{table}

    \section{Conclusion}
\label{sec:conclusion}
This paper presented \methodname{}, a gradient-flow method for trajectory optimization of robots with prescribed contact sequences.
The method combines pseudospectral collocation with a two-time-scale PHR augmented-Lagrangian flow whose equilibria satisfy the KKT conditions of the transcribed problem.
B\'ezier wrench parameterizations certify convex contact-wrench constraints throughout each interval when their control points are feasible.
A linearly implicit Gauss--Newton integrator preserves descent at fixed multipliers and penalties, while feasibility initialization and adaptive mesh refinement address initial infeasibility and violations between collocation nodes.

Across ten benchmarks on five robot platforms, \methodname{} meets the terminal tracking criteria in all ten simulated rollouts, compared with seven for Crocoddyl and eight for Aligator.
It achieves substantial solve-time reductions on several benchmarks and lower normalized terminal error over a range of compute budgets.
A hopscotch execution on the Unitree Go2 further demonstrates hardware tracking of an optimized reference, with a joint-position RMS error of $0.0682$~rad and centimeter-scale RMS position errors.
We will release all code upon final acceptance of this paper

    % Bibliography
    {\footnotesize
    \bibliographystyle{plainnat}
    \bibliography{references}
    }
\end{document}